\documentclass[letterpaper, 10 pt, conference]{ieeeconf}  

\IEEEoverridecommandlockouts                              

\usepackage{graphics} 
\usepackage{subfig} 
\usepackage{wrapfig}

\usepackage{amsthm} 
\usepackage{amsmath,amssymb}
\usepackage{graphicx}
\usepackage{tabularx}
\usepackage{multirow, makecell}
\usepackage{diagbox}
\usepackage{slashbox}
\usepackage{rotating}
\usepackage{cite}

\usepackage{url}
\usepackage{bm}
\usepackage[table]{xcolor}
\usepackage{hyperref}
\usepackage{siunitx} 
\usepackage{booktabs}
\usepackage{cleveref}
\usepackage{mathtools} 
\usepackage{upgreek} 

\let\labelindent\relax 
\usepackage{enumitem} 

\usepackage[ruled,vlined]{algorithm2e}

\usepackage[nolist, nohyperlinks]{acronym}
\acrodef{MPC}[MPC]{Model Predictive Control}
\acrodef{QP}[QP]{Quadratic Program}
\acrodef{CBF}[CBF]{Control Barrier Function}

\newcommand{\vx}{{\boldsymbol x}}
\newcommand{\vu}{{\boldsymbol u}}
\newcommand{\vq}{{\boldsymbol q}}

\newcommand{\vp}{{\boldsymbol p}} 
\newcommand{\vv}{{\boldsymbol v}} 
\newcommand{\vtau}{{\boldsymbol \tau}}

\newcommand{\lieder}{L}
\newcommand{\StateSpace}{\mathcal{X}}
\newcommand{\ControlSpace}{\mathcal{U}}
\newcommand{\RealSpace}{\mathbb{R}}

\newcommand{\calC}{\mathcal{C}} 
\newcommand{\calK}{\mathcal{K}} 

\newcommand{\calS}{\mathcal{S}} 

\newtheorem{definition}{Definition}
\newtheorem{theorem}{Theorem}

\theoremstyle{definition}
\newtheorem{remark}{Remark} 
\theoremstyle{definition}

\theoremstyle{definition}

\theoremstyle{definition}
\newtheorem{example}{Example}

\theoremstyle{definition}
\newtheorem{assumption}{Assumption}

\makeatother

\DeclareCaptionFont{mysize}{\fontsize{8}{9.6}\selectfont}
\title{\LARGE \bf Is Your Safe Controller Actually Safe? \\ A Critical Review of CBF Tautologies and Hidden Assumptions} 

\author{Taekyung Kim
\thanks{Department of Robotics, University of Michigan, Ann Arbor, MI, 48109, USA {\tt\footnotesize taekyung@umich.edu} }  %
}

\begin{document}
\maketitle
\thispagestyle{empty}
\pagestyle{empty}

\begin{abstract}
This tutorial provides a critical review of the practical application of Control Barrier Functions (CBFs) in robotic safety. While the theoretical foundations of CBFs are well-established, I identify a recurring gap between the mathematical assumption of a safe controller's existence and its constructive realization in systems with input constraints. I highlight the distinction between candidate and valid CBFs by analyzing the interplay of system dynamics, actuation limits, and class-$\calK$ functions. I further show that some purported demonstrations of safe robot policies or controllers are limited to passively safe systems, such as single integrators or kinematic manipulators, where safety is already inherited from the underlying physics and even naive geometric hard constraints suffice to prevent collisions. By revisiting simple low-dimensional examples, I show when CBF formulations provide valid safety guarantees and when they fail due to common misuses. I then provide practical guidelines for constructing realizable safety arguments for systems without such passive safety. A crowd-navigation simulation study further illustrates that CBF-derived reward shaping in reinforcement learning can improve empirical behavior without establishing formal safety. The goal of this tutorial is to bridge the gap between theoretical guarantees and actual implementation, supported by an open-source interactive web demonstration that visualizes these concepts intuitively.
 \href{https://github.com/tkkim-robot/safe_control}{\textcolor{red}{[Code]}} \href{https://cbf.taekyung.me}{\textcolor{red}{[Web Demo]}}   
\end{abstract}


\section*{Preliminaries}
Consider a control-affine dynamical system with state $\vx \in \StateSpace \subset \mathbb{R}^{n}$ and control input $\vu \in \ControlSpace \subset \mathbb{R}^{m}$, where $\ControlSpace$ encodes admissible actuation. The dynamics are
\begin{equation} \label{eq:dynamics}
    \dot{\vx} = f(\vx) + g(\vx)\vu.
\end{equation}
Unless stated otherwise, I assume $f$ and $g$ are locally Lipschitz on $\StateSpace$ and that $\ControlSpace$ is nonempty. Consequently, for any locally Lipschitz feedback policy $\pi : \StateSpace \to \ControlSpace$, the closed-loop system admits well-defined solutions.

I denote by $\calS$ the set of \textbf{\emph{instantaneously safe}} states, i.e., states that satisfy the hard safety specifications (e.g., geometric collision-free constraints, joint limits). A \emph{safe set} (or \emph{controlled-invariant set}) is any set $\calC \subseteq \calS$ that the closed-loop system can be kept inside for all future time.

\begin{definition}[Forward Invariance]
    A set $\calC$ is \emph{forward invariant} under the closed-loop dynamics if for every initial condition $\vx(0) \in \calC$, the solution satisfies $\vx(t) \in \calC$ for all $t \ge 0$.
\end{definition}

\section{Tautological Safety Guarantees}

A recurring failure mode in safety-critical control is a result that is mathematically correct yet algorithmically non-constructive: the safety property is effectively assumed in the hypothesis and then restated as the conclusion. A stripped-down version of this pattern can be written as follows.

\begin{assumption}[Existence of a Safe Controller]
\label{assumption:existence}
    There exists a control policy $\pi_{s} : \StateSpace \to \ControlSpace$ such that for any initial condition $\vx_{0} \in \calC$, the closed-loop trajectories satisfy $\vx(t) \in \calC$ for all $t \ge 0$.
\end{assumption}

\begin{figure}[t]
    \centering
    \includegraphics[width=0.9\columnwidth]{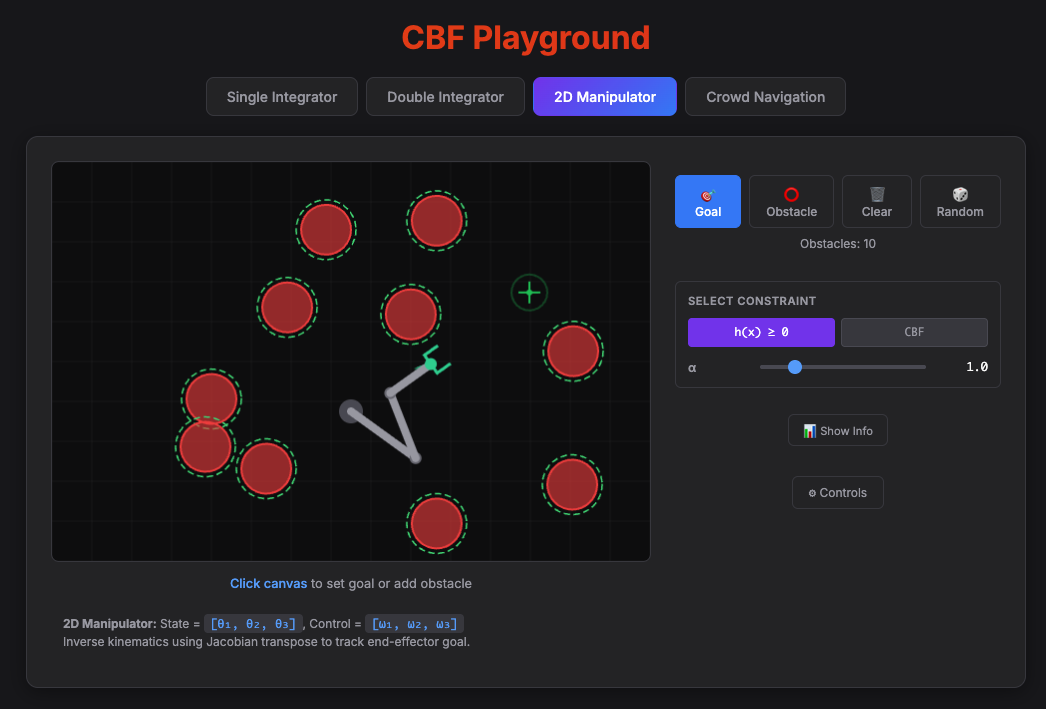}
    \caption{\url{https://cbf.taekyung.me}: interactive web demonstration of the CBF Playground. The demo contrasts \emph{passively safe} (driftless) systems (single integrator, kinematic manipulator), for which naive geometric constraints already prevent collisions, and a system with \emph{inertia} (double integrator), where feasibility depends on actuation limits, velocity bounds, and class-$\calK$ tuning.
    }
    \label{fig:online_demo}
\end{figure}

\begin{theorem}[Safety Guarantee]
\label{thm:safety-guarantee}
    Under \autoref{assumption:existence}, the closed-loop system controlled by $\pi_{s}$ renders the set $\calC$ forward invariant.
\end{theorem}

\begin{proof}
Immediate from the definition of forward invariance.
\end{proof}

\autoref{thm:safety-guarantee} is intentionally vacuous: it formalizes the statement ``the system is safe if there exists a controller that keeps it safe.'' The reason to spell out this tautology is that the same logical structure often reappears in modern presentations of Control Barrier Functions (CBFs) when the \emph{verification of hypotheses} is skipped. To keep the discussion precise, I restate the standard CBF definition and theorem.

\begin{definition}[CBF~\cite{ames_control_2019}] \label{def:cbf}
   Let $h: \StateSpace \to \RealSpace$ be continuously differentiable and define the set $\calC \coloneqq \{\vx \in \StateSpace \mid h(\vx) \geq 0\}$. The function $h$ is a \emph{control barrier function} (CBF) for System~\eqref{eq:dynamics} on a set $D \subseteq \StateSpace$ if there exists an extended class $\calK_{\infty}$ function $\alpha$ such that
   \begin{equation}
       \sup_{\vu \in \ControlSpace} \left[\lieder_{f} h(\vx) + \lieder_{g} h(\vx)\vu\right] \geq -\alpha\big(h(\vx)\big) \quad \forall \vx \in D,
   \label{eq:cbf-condition}
   \end{equation}
   where $\lieder_{f} h(\vx)$ and $\lieder_{g} h(\vx)$ denote the Lie derivatives of $h$ along $f$ and $g$, respectively.
\end{definition}

Given a CBF~$h$ and $\alpha$, the set of admissible control inputs at state $\vx$ is
\begin{equation}
K_{\textup{cbf}}(\vx; \alpha) \coloneqq \left\{\vu \in \ControlSpace \mid \lieder_{f} h(\vx) + \lieder_{g} h(\vx)\vu \geq -\alpha\big(h(\vx)\big) \right\}. \nonumber
\label{eq:k_cbf}
\end{equation}

\begin{theorem}[\cite{ames_control_2019}]
\label{thm:cbf-forward-invariance}
   Let $\calC = \{\vx \mid h(\vx) \ge 0\}$ with $h \in C^{1}$ and assume $\nabla h(\vx) \neq 0$ for all $\vx \in \partial \calC$. If $h$ is a CBF on a set $D$ satisfying $\calC \subseteq D \subseteq \StateSpace$, then any locally Lipschitz controller $\vu(\vx) \in K_{\textup{cbf}}(\vx;\alpha)$ for all $\vx \in D$ renders $\calC$ forward invariant for System~\eqref{eq:dynamics}.
\end{theorem}

The key point is \emph{where the burden of proof sits}. \autoref{thm:cbf-forward-invariance} is correct, but it is conditional on:
\begin{enumerate}
    \item showing that the proposed $h$ is \emph{actually} a CBF, i.e., that \eqref{eq:cbf-condition} holds on an appropriate domain under the \emph{true} input bounds $\ControlSpace$; and
    \item ensuring well-posedness (e.g., local Lipschitz continuity of the implemented feedback law)~\cite{agrawal_reformulations_2025}.
\end{enumerate}
If these hypotheses are assumed rather than verified, then the ``CBF safety guarantee'' collapses back to the tautology of \autoref{thm:safety-guarantee}: it is safe because I assumed the existence of a controller that keeps it safe.

To make the common gap explicit, I distinguish between a \emph{candidate} CBF and a \emph{valid} CBF.

\begin{definition}[Candidate CBF~\cite{kim_how_2025}]\label{def:candidate_cbf}
Let $\tilde{h}: \StateSpace \to \RealSpace$ be continuously differentiable and define $\tilde{\calC} \coloneqq \{\vx \in \StateSpace \mid \tilde{h}(\vx) \geq 0\}$. Any such $\tilde{h}$ with $\tilde{\calC} \subseteq \calS$ is called a \emph{candidate} CBF.
\end{definition}

In many robotics applications, a candidate $\tilde{h}$ is chosen from geometry (often, $\tilde{\calC}=\calS$). This is a reasonable starting point, but it does \emph{not} by itself justify the use of \autoref{thm:cbf-forward-invariance}. The missing step is to verify that the CBF feasibility condition holds everywhere on the claimed safe set:
\begin{equation}
K_{\textup{cbf}}(\vx;\alpha) \neq \emptyset \quad \forall \vx \in \tilde{\calC}.
\label{eq:cbf-pointwise-feas}
\end{equation}
When \eqref{eq:cbf-pointwise-feas} fails, any downstream ``CBF-QP'' controller is undefined at those states, and the formal safety claim does not apply.

\section{Defining Safety and Common Fallacies}

Before discussing how to construct valid certificates, I formalize what I mean by a \emph{safe} system and separate it from weaker notions that frequently appear in practice.

\subsection{Safety as Forward Invariance}

\begin{definition}[Safe Autonomous System]
\label{def:safe-system}
    Consider the closed-loop dynamics $\dot{\vx} = f(\vx) + g(\vx)\pi_{s}(\vx)$, where $\pi_{s}:\StateSpace\to\ControlSpace$ is locally Lipschitz. The system is a \emph{safe autonomous system} with respect to $\calC$ if $\calC$ is forward invariant under the closed-loop dynamics, i.e., for all $\vx_{0} \in \calC$, the solution satisfies $\vx(t) \in \calC$ for all $t \ge 0$.
\end{definition}

In set-based approaches, $\calC$ is typically represented as the superlevel set of a scalar function $h$. The challenge is not writing down a candidate $\tilde{h}$, but proving (or at least rigorously validating) that it satisfies the derivative inequality required for a CBF\footnote{I strictly distinguish between a \textit{candidate} CBF (an arbitrary differentiable constraint function) and a \textit{CBF} (a function verified to satisfy \eqref{eq:cbf-condition}).}.

\subsection{Distinguishing Empirical Success from Formal Safety}\label{sec:empirical_safety}

\subsubsection{Empirical validation is not safety.}
Finite-horizon experiments (in simulation or hardware) can demonstrate performance and reveal failure modes, but they do not establish forward invariance. Even a system that succeeds for $10^{6}$ trials\footnote{In practice, demonstrations often rely on even fewer trials, typically in the range $10^{1}$ to $10^{2}$.} can fail on the next rollout if the invariance property is not proven.

\subsubsection{Safety-informed is not safe.}
Safety is a \emph{hard constraint}: unsafe states are inadmissible. In contrast, many learning and optimization pipelines encode safety as a \emph{soft penalty} that can be traded off against reward or performance.

Common fallacies include~\cite{brunke_safe_2022}:
\begin{itemize}
    \item \textbf{Cost-function penalties:} Adding terms proportional to violation of an analytical safety condition to the objective (e.g., penalizing $\max\{0, -\dot{h}(\vx) - \alpha(h(\vx))\}$).
    \item \textbf{Neural CBF penalties:} Using a learned candidate barrier $V_{\theta}(\vx)$ as a penalty (e.g., $\max\{0, -\dot{V}_{\theta}(\vx) - \alpha(V_{\theta}(\vx))\}$).
    \item \textbf{RL reward shaping:} Assigning negative reward upon constraint violation. This encourages safe behavior but does not remove unsafe actions from the admissible control set.
\end{itemize}

These methods produce systems that are at best \emph{safety-informed}. Because safety can be overridden by competing objectives, they cannot, by themselves, claim the forward-invariance guarantee in \autoref{def:safe-system}.

\subsubsection{Hard constraints without feasibility.}
Encoding safety as a hard constraint in a QP/MPC/Constrained MDP does not guarantee safety unless the optimization problem remains solvable for all time. If the optimization becomes infeasible at some time $t$, the controller is undefined and the safety argument breaks. Guaranteeing feasibility \emph{for all future time} is the problem of \emph{recursive feasibility}~\cite{borrelli_predictive_2017}, and it is frequently neglected.

\subsubsection{Well-posedness is part of the safety claim.}
Even when a pointwise feasible set $K_{\textup{cbf}}(\vx;\alpha)$ exists, the safety theorem requires that the implemented feedback be well-defined (typically, locally Lipschitz) so that trajectories exist and are unique~\cite{agrawal_reformulations_2025}. In practice, many CBF controllers are implemented via parametric QPs; active-set switches, degeneracy, and numerical tolerances can introduce nonsmooth (or discontinuous) feedback laws. When these issues arise, ``safety'' becomes as much a numerical-analysis problem as a control-theoretic one.

\section{The Challenge of Recursive Feasibility}

To understand why the existence assumption in \autoref{assumption:existence} is so pernicious, it is helpful to look at a deliberately extreme counterexample.

\begin{example}[Double Integrator]
Consider the (1D) double integrator with position $p$ and velocity $v$:
\begin{equation}
    \dot{\vx} = \begin{bmatrix} \dot{p} \\ \dot{v} \end{bmatrix} = \begin{bmatrix} 0 & 1 \\ 0 & 0 \end{bmatrix}\vx + \begin{bmatrix} 0 \\ 1 \end{bmatrix}u, \qquad |u| \le 1.
\end{equation}
Let the safety constraint be a wall at $p=0$, i.e., $p(t) \ge 0$.

Now initialize at the boundary with negative velocity $\vx(0) = [0, -c]^{\top}$, where $c = 299{,}792{,}458$ m/s (the speed of light). Under \autoref{assumption:existence}, the system is ``safe.'' In reality, safety is impossible: for any admissible input $|u|\le 1$, the position satisfies $p(t) = 0 + (-c)t + \tfrac{1}{2}u t^{2} < 0$ for all sufficiently small $t>0$, so the state exits the safe set immediately.
\end{example}
The key point is structural: for a double integrator, \emph{any} initial condition on the boundary with $v(0)<0$ violates $p(t)\ge 0$ instantaneously.

Equivalently, the geometric set $\{(p,v)\mid p\ge 0\}$ is \emph{not} controlled invariant under bounded acceleration. The maximal controlled-invariant subset is strictly smaller and couples position and velocity. For the 1D double integrator with $|u|\le u_{\max}$ and constraint $p\ge 0$, a classical calculation yields the viability condition $p \ge v^{2}/(2u_{\max})$ when $v<0$, i.e., $v \ge -\sqrt{2u_{\max}p}$.

This is where the appeal of CBFs originates: a valid CBF condition provides a local infinitesimal certificate of forward invariance. However, if one does not prove that a proposed $\tilde{h}$ actually satisfies \eqref{eq:cbf-condition} under the true $\ControlSpace$ on the claimed domain, then the ``safety guarantee'' remains an unverified hypothesis.

\section{Constructing Safety: Tuning the Inequality \label{sec:tuning}}

Having established that \emph{assuming} the existence of a CBF is logically empty, I now turn to the practical question: how does one verify that a candidate function is a valid CBF for a given system and input limits?

Recall the CBF feasibility condition:
\begin{equation}
  \sup_{\vu \in \ControlSpace} \left[ \lieder_{f} h(\vx) + \lieder_{g} h(\vx)\vu \right] \geq -\alpha\big(h(\vx)\big).
\end{equation}
To reason about feasibility, it is useful to separate the inequality into its core design/plant components:
\begin{enumerate}
 \item the dynamics $(f, g)$,
 \item the input bounds ($\ControlSpace$),
 \item the constraint representation ($h$),
 \item the class-$\calK$ function ($\alpha$),
 \item the domain on which safety is claimed.
\end{enumerate}

\subsection{Fixed Components: Dynamics and Actuation}
The first two components, the dynamics $(f, g)$ and the admissible input set $\ControlSpace$, are physical properties of the plant (up to modeling error). They determine the maximum authority available to counter unsafe drift. Relaxing $\ControlSpace$ (e.g., pretending $\ControlSpace=\mathbb{R}^{m}$) can make \eqref{eq:cbf-condition} appear feasible in theory, but yields controllers that demand physically impossible actuation.

\subsection{Choosing the Set via $h$}
The choice of $h$ defines the geometry of $\calC$. When a hand-designed candidate fails \eqref{eq:cbf-condition}, it often indicates that the intended set is too large: it contains states from which recovery is impossible under $\ControlSpace$.

In low-dimensional settings, constructive methods can compute the maximal safe set. Hamilton--Jacobi reachability can compute the viability kernel~\cite{altarovici_general_2013, choi_robust_2021}, and Sum-of-Squares (SOS) programming can provide polynomial inner approximations~\cite{parrilo_semidefinite_2003, xu_correctness_2018}. In higher dimensions, practitioners often rely on heuristic candidates and then tune the remaining ingredients.

\subsection{Tuning Class-$\calK$ Functions}
The function $\alpha$ dictates the allowable rate of approach to the safety boundary. Tuning $\alpha$ presents a fundamental trade-off~\cite{parwana_ratetunable_2025, kim_learning_2025, kim_how_2025}:
\begin{itemize}
    \item \textbf{Aggressive Tuning:} If $\alpha$ is steep (e.g., $\alpha(h) = \gamma h$ with large $\gamma > 0$), the controller allows the system to approach the boundary rapidly. However, stopping at the last moment requires significant control authority. This makes the CBF condition \emph{harder} to satisfy, as the required $\vu$ might lie outside $\ControlSpace$.
  \item \textbf{Conservative Tuning:} If $\alpha$ is shallow, the system is forced to slow down far from the unsafe set. This lowers the required control effort, making the CBF condition \emph{easier} to satisfy. However, this results in overly conservative behavior where the robot refuses to move near obstacles even when it is safe to do so.
\end{itemize}

\subsection{Restricting the Operating Domain}
Finally, feasibility can be recovered by restricting the domain on which safety is claimed. For systems with inertia, the CBF inequality often fails on unbounded state spaces because the drift can be arbitrarily unsafe.

\begin{example}[Bounded-Velocity Double Integrator]
Revisit the double integrator with $\vx = [p, v]^{\top}$ and the geometric constraint $p \ge 0$. A naive candidate barrier is $h(\vx)=p$. Since $\lieder_{g}h(\vx)=0$, the standard first-order CBF condition is ill-posed (relative degree two), and one must use a high-order construction so that the input appears.

If the domain is unbounded ($\StateSpace=\mathbb{R}^{2}$), then $v$ can be arbitrarily negative, corresponding to arbitrarily large unsafe drift toward the wall. Under bounded actuation, no certificate can make the set $\{p\ge 0\}$ invariant on $\mathbb{R}^{2}$. However, if one restricts the domain by enforcing $v \in [v_{\min}, v_{\max}]$ (or, more tightly, $v \ge -\sqrt{2u_{\max}p}$), then feasibility can be recovered on the restricted domain.
\end{example}

Well-illustrated examples of domain restriction to ensure feasibility can be found in \cite{park_collision_2026}.

\subsection{Using CBFs Inside Planning}
Unlike purely reactive controllers, global planners can search the configuration space for trajectories that avoid regions where the CBF inequality is infeasible. One strategy is to use the CBF condition as a validity check within a planner.

For example, \cite{kim_visibilityaware_2025} uses a CBF inequality to filter sampled states or edges during tree expansion; if a sample violates the CBF condition, it is excluded. This biases the planner toward trajectories that are certifiable under the downstream controller (assuming tracking can be performed with sufficiently small error). The trade-off is completeness: if the CBF is conservative or improperly tuned, the planner may report ``no solution'' even when a collision-free path exists in the purely geometric free space.

\section{When Safety is Structurally Trivial: Passively Safe Systems}

To interpret many empirical ``safety'' demonstrations, it is useful to identify systems for which safety is structurally guaranteed by physics/modeling, rendering sophisticated certification superfluous.

\subsection{Passively Safe Systems}

\begin{definition}[Passively Safe System]
\label{def:passive-safe}
A control-affine system $\dot{\vx} = f(\vx) + g(\vx)\vu$ is \emph{passively safe} with respect to a set $\calC$ if $\calC$ is forward invariant under the \emph{unforced} dynamics $\dot{\vx}=f(\vx)$.\footnote{The notion that a system remains safe under zero input is widely referred to as \emph{passive safety}~\cite{breger_safe_2012, starek_sampling_2015}.}
\end{definition}

For sets represented as $\calC = \{\vx \mid h(\vx)\ge 0\}$ with $h\in C^{1}$, a sufficient condition for passive safety is
\begin{equation}
    \lieder_{f}h(\vx) \ge 0 \quad \forall \vx \in \calC .
\end{equation}

\begin{theorem}[Trivial Safety]
\label{thm:trivial-safety}
    If the system is passively safe with respect to $\calC$ and the admissible control set contains the origin ($\mathbf{0} \in \ControlSpace$), then $\calC$ is forward invariant under the feedback $\vu(\vx) \equiv \mathbf{0}$.
\end{theorem}

\begin{proof}
Choose $\vu(\vx) \equiv \mathbf{0}$. The closed-loop dynamics reduce to $\dot{\vx}=f(\vx)$, under which $\calC$ is forward invariant by \autoref{def:passive-safe}.
\end{proof}

\subsection{Driftless Systems and Relative Degree}

A prominent subclass of passively safe systems is the class of \emph{driftless} systems, where $f(\vx)=\mathbf{0}$. In this case, the unforced dynamics are stationary, so any set $\calC$ is forward invariant under $\vu\equiv\mathbf{0}$.

\begin{example}[$n$-Dimensional Single Integrator]
Consider the single integrator $\dot{\vx} = \vu$ with $\vx,\vu \in \RealSpace^{n}$. Here, $f(\vx)=\mathbf{0}$ and $g(\vx)=I_{n}$. For any differentiable constraint function $h$,
\begin{align}
    \dot{h}(\vx,\vu)
    &= \frac{\partial h}{\partial \vx}(f(\vx)+g(\vx)\vu) 
    = \underbrace{\frac{\partial h}{\partial \vx}f(\vx)}_{\lieder_{f}h(\vx)=0} + \lieder_{g}h(\vx)\vu. \nonumber
\end{align}
Since $\lieder_{f}h(\vx)=0$, the unforced dynamics never decrease $h$.
\end{example}

\begin{example}[Kinematic Manipulators]
Consider a robotic manipulator with configuration $\vq \in \mathbb{R}^{n}$ and Euler--Lagrange dynamics
\begin{equation}
M(\vq)\ddot{\vq} + C(\vq, \dot{\vq})\dot{\vq} + G(\vq) = \vtau.
\end{equation}
In many manipulation stacks, a high-rate low-level controller tracks a commanded joint velocity $\vu \in \mathbb{R}^{n}$, yielding the kinematic model
\begin{equation}
\dot{\vq} = \vu.
\end{equation}
Defining the state as $\vx=\vq$, this is driftless: $f(\vx)=\mathbf{0}$, $g(\vx)=I_{n}$. Hence it is passively safe under the kinematic abstraction.
\end{example}

\begin{remark}[Dynamic Environments]
Even for driftless robots, safety becomes nontrivial when the safe set is time-varying (dynamic obstacles). If $h(\vx,t)$ defines $\calC(t)$, then
\begin{equation}
    \dot{h}(\vx,\vu,t) = \frac{\partial h}{\partial \vx}g(\vx)\vu + \frac{\partial h}{\partial t}.
\end{equation}
The term $\tfrac{\partial h}{\partial t}$ acts as an uncontrollable \emph{virtual drift}. Safety then depends on explicit assumptions on the environment (e.g., bounded obstacle speed) and on whether the available control authority can overcome this drift.
\end{remark}

\begin{remark}[The Trap of Inertia]
Systems with relative degree $\ge 2$, such as the double integrator $\ddot{\vp}=\vu$, are sometimes informally described as ``simple'' because the input enters linearly. However, in state space $\vx=[\vp,\vv]^{\top}$ with $\dot{\vp}=\vv$ and $\dot{\vv}=\vu$, the drift term is $f(\vx)=[\vv,\mathbf{0}]^{\top}\neq \mathbf{0}$. If $h$ depends only on position $\vp$, then $\lieder_{f}h(\vx)$ depends on velocity $\vv$ and can be arbitrarily negative on an unbounded domain, making feasibility nontrivial (or impossible) under bounded actuation.
\end{remark}

\begin{example}[Artificial Potential Fields]
Artificial Potential Field (APF) methods are classically known to provide collision avoidance for driftless point robots. Let $\phi: \StateSpace \to \RealSpace$ be a potential function that diverges near obstacle boundaries, and apply $\vu=-\nabla\phi(\vx)$. For the driftless model $\dot{\vx}=\vu$, the control directly shapes the state velocity, enabling repulsion from obstacles~\cite{singletary_comparative_2021}. This intuition does \emph{not} transfer unchanged to systems with inertia or input saturation, where $\vu$ cannot instantaneously realize the desired state velocity.
\end{example}

\begin{remark}
When the safety output has relative degree $>1$, $\lieder_{g} h(\vx)=0$ and the first-order inequality is ill-posed. I refer interested readers to High-Order CBFs (HOCBF)~\cite{xiao_control_2019} and Input-Constrained CBFs (ICCBF)~\cite{agrawal_safe_2021}. For discrete-time counterparts, see \cite{xiong_discretetime_2023} for HOCBF and \cite{kim_learning_2025} for ICCBF.
\end{remark}

\section{Practical Guidelines for Safety Arguments}

The discussion above suggests a simple checklist that helps prevent over-claiming safety.

\subsection{A Common Generalization Gap}

\begin{example}[A Flawed Safety Argument]
A high-level reduction of many flawed safety arguments is:
\begin{itemize}
    \item \textbf{Step 1 (Theory):} State results for a general control-affine model $\dot{\vx} = f(\vx) + g(\vx)\vu$.
    \item \textbf{Step 2 (Certificate):} Propose a geometric candidate $\tilde{h}(\vx)=\|\vx-\vx_{\textup{obs}}\|^{2}-R^{2}$.
    \item \textbf{Step 3 (Controller):} Implement a CBF-QP using $\tilde{h}$.
    \item \textbf{Step 4 (Claim):} Conclude ``therefore, our system is safe.''
\end{itemize}
Then, in experiments: ``we model our robot as a single integrator (or kinematic manipulator).''
\end{example}

\textit{Critique:} The candidate $\tilde{h}$ is typically only a valid certificate for driftless or relative-degree-one abstractions. If the experimental model is passively safe while the claimed theory targets a general inertial system, then the reported ``safety'' does not validate the advertised claim.

\subsection{Best Practices}

\begin{itemize}
    \item \textbf{State the safety claim precisely:} Specify the set $\calC$, the domain on which safety is claimed, and whether safety is continuous-time or sampled-data.
    \item \textbf{Separate candidates from certificates:} If a function is only heuristically motivated, call it a candidate and avoid invoking \autoref{thm:cbf-forward-invariance}.
    \item \textbf{Check feasibility under input bounds:} Explicitly argue (or compute) that $K_{\textup{cbf}}(\vx;\alpha)$ is nonempty on the relevant domain.
\end{itemize}

\section{Simulation Results}
\label{sec:simulation}

To empirically support the discussion, I present simulations on three dynamical systems: (i) single integrator, (ii) double integrator, and (iii) 3-link planar manipulator under a kinematic abstraction. I compare two constraint realizations in a QP-based controller:
\begin{enumerate}
    \item \textbf{CBF-QP:} The standard formulation enforcing $\dot{h}(\vx) \ge -\alpha\big(h(\vx)\big)$ (relative degree one) or a high-order equivalent for the double integrator.
    \item \textbf{Naive hard constraint:} A baseline enforcing a discrete-time forward prediction constraint. For relative degree one: $h(\vx_{k+1}) \ge 0$. For relative degree two: $h(\vx_{k+2}) \ge 0$ using a second-order Taylor approximation. The objective function is identical to that of the CBF-QP controller.
\end{enumerate}

\subsection{Experimental Setup}

\subsubsection{Implementation Details}
All simulations are implemented in Python using a discrete-time simulation with time step $\Delta t = 0.01$ s. The QP-based controllers are solved using \texttt{cvxpy}~\cite{diamond_cvxpy_2016} with the Gurobi solver. 
\footnote{Source code is available at \url{https://github.com/tkkim-robot/safe_control}.}

\begin{figure*}[t]
    \centering
    \subfloat[Single Integrator]{\includegraphics[width=0.32\textwidth]{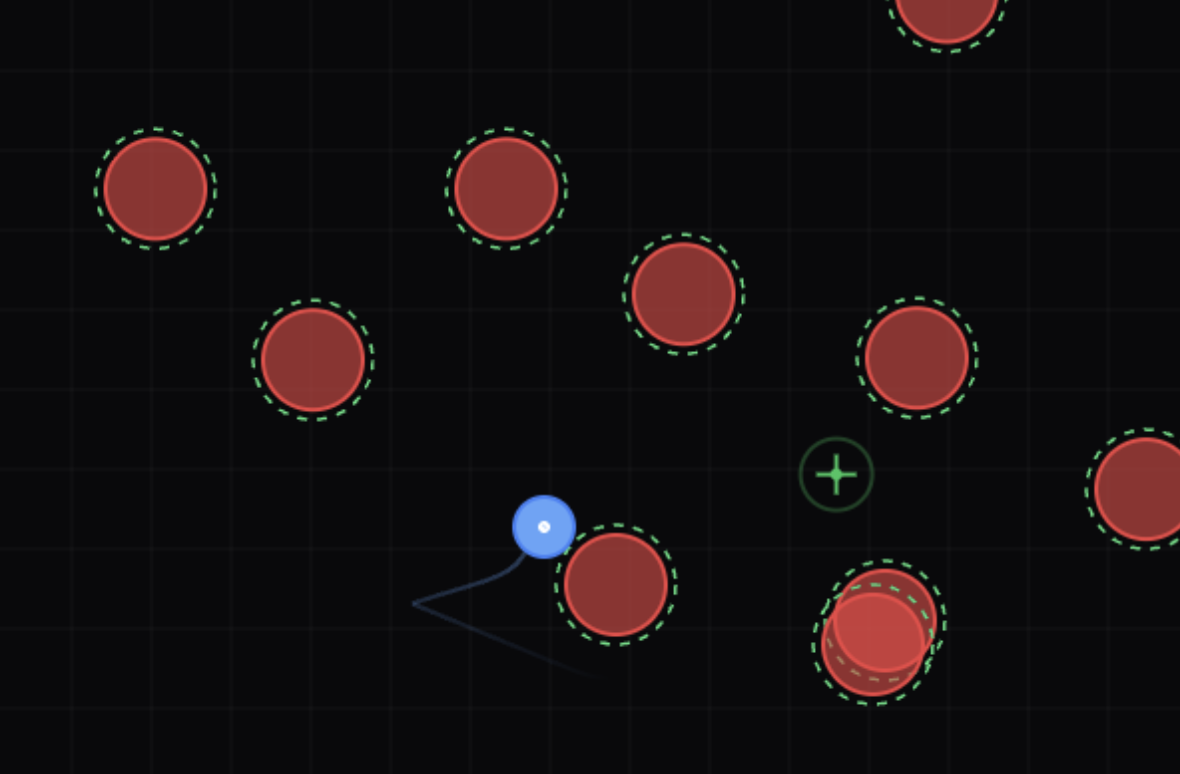}}\hfill
    \subfloat[Double Integrator]{\includegraphics[width=0.32\textwidth]{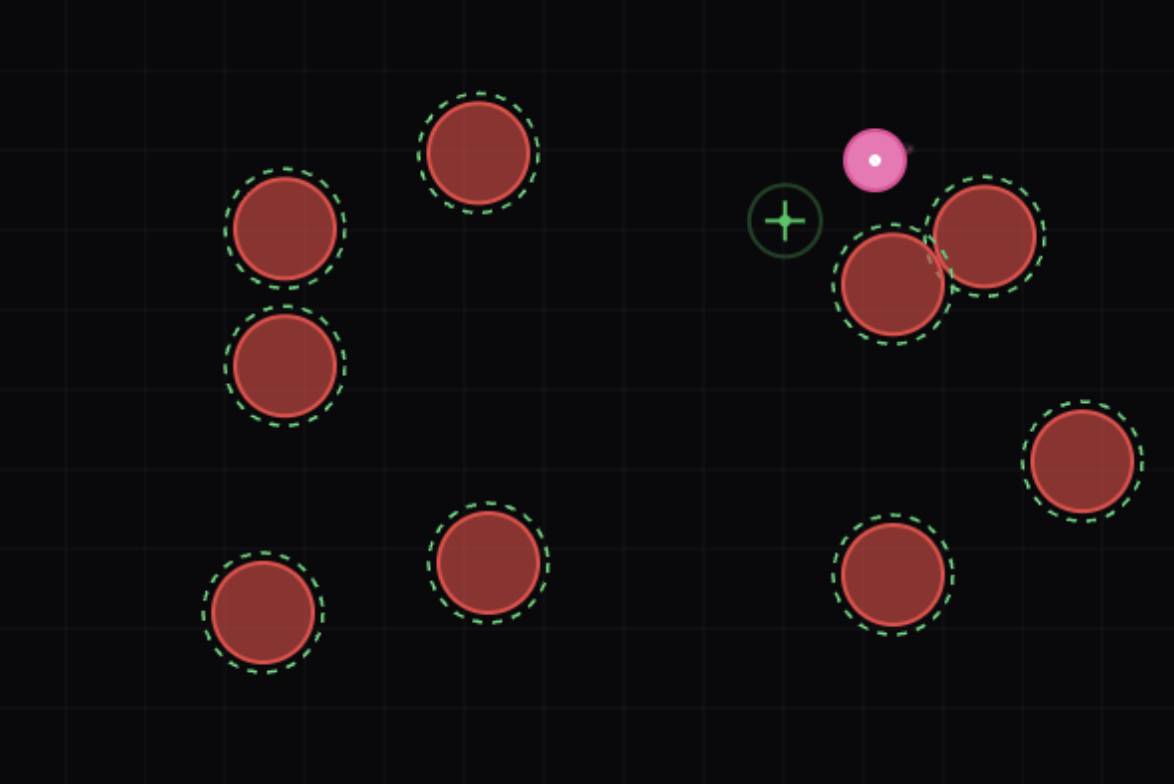}}\hfill
    \subfloat[Planar Manipulator]{\includegraphics[width=0.32\textwidth]{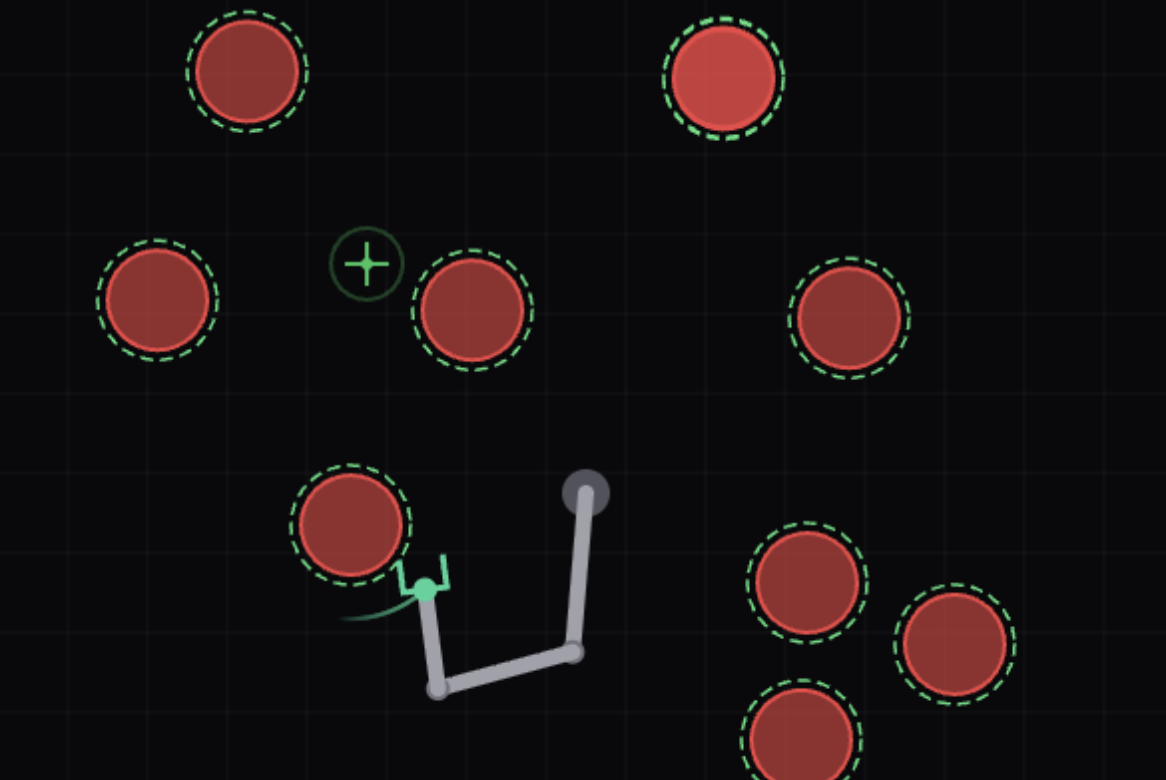}}
    \caption{Snapshots of simulation trials. Red circles are obstacles and the green cross is the goal. Interested readers may evaluate the simulation framework via the provided live web demo.}
    \label{fig:experiments}
\end{figure*}

\subsubsection{Environment Configuration}
The simulations are conducted in a 10 m $\times$ 10 m workspace. For each trial:
\begin{itemize}
    \item \textbf{Obstacles:} 10 circular obstacles are randomly placed, each with radius $R_{\text{obs}} = 0.4$ m. Obstacles are validated to ensure they do not overlap with the robot's initial configuration.
    \item \textbf{Robot Radius:} The robot is modeled as a disk with radius $R_{\text{robot}} = 0.25$ m for collision checking.
    \item \textbf{Trials:} 100 randomized trials per configuration, with different initial positions and goal locations.
\end{itemize}

\subsubsection{System Parameters}
\begin{itemize}
    \item \textbf{Single Integrator:} Maximum velocity $v_{\max} = 5.0$ m/s. Linear class-$\calK$ gains $\alpha \in \{1.0, 2.0, 3.0, 4.0, 5.0\}$.
    \item \textbf{Double Integrator:} Maximum acceleration $a_{\max} = 5.0$ m/s$^{2}$. Grid over $v_{\max} \in \{1.0, 2.0, 3.0, 4.0, 5.0\}$ m/s and HOCBF gains $\alpha_{1}=\alpha_{2} \in \{1.0, 2.0, 3.0, 4.0, 5.0\}$.
    \item \textbf{3-DOF Manipulator:} Link lengths $L = [1.33, 1.16, 0.83]$ m. Maximum joint velocity $\omega_{\max} = 2.0$ rad/s. Linear class-$\calK$ gains $\alpha \in \{1.0, 2.0, 3.0, 4.0, 5.0\}$.
\end{itemize}

\subsection{Results and Discussion}

\subsubsection{Passively Safe Systems}
Consistent with \autoref{thm:trivial-safety}, both the single integrator and the kinematic manipulator exhibit \emph{0\% collisions and 0\% infeasibility} across all 100 trials for \emph{all} constraint types and parameter values tested. Under these driftless abstractions, $\vu=\mathbf{0}$ keeps the state fixed, so \textbf{\emph{even naive hard constraints appear successful}}.

\begin{table}[t]
\caption{Passively safe systems (100 trials each)}
\label{tab:results_safe}
\centering
\begin{tabular}{lcc}
\toprule
\textbf{System / constraint family} & \textbf{Col. Rate} & \textbf{Inf. Rate} \\
\midrule
Single Integrator: CBF ($\alpha \in \{1,\dots,5\}$) & 0\% & 0\% \\
Single Integrator: Naive hard constraint & 0\% & 0\% \\
3-DOF Manipulator: CBF ($\alpha \in \{1,\dots,5\}$) & 0\% & 0\% \\
3-DOF Manipulator: Naive hard constraint & 0\% & 0\% \\
\bottomrule
\end{tabular}
\end{table}

\subsubsection{Failure of the Naive Hard Constraint for Inertial Systems}
For the double integrator, the naive prediction constraint fails catastrophically, with collision rates of 87--93\% across all tested velocity bounds (\autoref{tab:double_hard}). This highlights the core message: geometric feasibility at the next step does not imply forward invariance for inertial dynamics.

\begin{table}[t]
\caption{Double integrator with naive hard constraint (100 trials each)}
\label{tab:double_hard}
\centering
\begin{tabular}{ccc}
\toprule
\textbf{$v_{\max}$} (m/s) & \textbf{Col. Rate} & \textbf{Inf. Rate} \\
\midrule
1.0 & 87\% & 0\% \\
2.0 & 91\% & 0\% \\
3.0 & 93\% & 0\% \\
4.0 & 90\% & 0\% \\
5.0 & 88\% & 0\% \\
\bottomrule
\end{tabular}
\end{table}

\subsubsection{Tuning CBFs Involves Trade-offs}
The CBF-QP controller exhibits a clear trade-off between allowed speed (system aggressiveness) and barrier tuning. As shown in \autoref{tab:double_grid}, conservative settings achieve 0\% collisions in our experiments, while more aggressive settings lead to higher collision rates and occasional infeasibility.

\begin{table}[t]
\caption{Double integrator: collision rate (\%) / infeasibility rate (\%) for different HOCBF gains}
\label{tab:double_grid}
\centering
\resizebox{\linewidth}{!}{%
\begin{tabular}{c cc cc cc cc cc}
\toprule
\multirow{2}{*}{$\alpha_{1}=\alpha_{2}$}
& \multicolumn{2}{c}{$v_{\max}=1$}
& \multicolumn{2}{c}{$v_{\max}=2$}
& \multicolumn{2}{c}{$v_{\max}=3$}
& \multicolumn{2}{c}{$v_{\max}=4$}
& \multicolumn{2}{c}{$v_{\max}=5$} \\
\cmidrule(lr){2-3}\cmidrule(lr){4-5}\cmidrule(lr){6-7}\cmidrule(lr){8-9}\cmidrule(lr){10-11}
& Col. & Inf. & Col. & Inf. & Col. & Inf. & Col. & Inf. & Col. & Inf. \\
\midrule
1.0 & 0 & 0 & 0 & 0 & 0 & 0 & 0 & 0 & 0 & 0 \\
2.0 & 0 & 0 & 0 & 0 & 0 & 0 & 0 & 0 & 0 & 0 \\
3.0 & 0 & 0 & 31 & 0 & 33 & 4 & 32 & 6 & 30 & 5 \\
4.0 & 17 & 0 & 68 & 1 & 68 & 4 & 63 & 5 & 60 & 8 \\
5.0 & 50 & 1 & 80 & 1 & 82 & 2 & 77 & 2 & 74 & 7 \\
\bottomrule
\end{tabular}
}
\end{table}

\textbf{Key observations:}
\begin{itemize}
    \item For $\alpha_{1}=\alpha_{2} \le 2$, the CBF-QP achieves \textbf{0\% collision and 0\% infeasibility} across all tested $v_{\max}$.
    \item Larger gains combined with larger velocities lead to significant collision rates (up to 82\%) and occasional infeasibility (up to 8\%).
\end{itemize}

\subsubsection{Discussion: Why ``Easy'' Systems are Misleading}
These results demonstrate that for passively safe abstractions (e.g., single integrators and kinematic manipulators), ``safety'' is often trivial: even simplistic geometric constraints appear to work because the drift does not push the system toward constraint violation. In contrast, for inertial systems like the double integrator, safety depends critically on feasibility under input bounds and on how the certificate couples position with momentum. Our experiments do \emph{not} prove that a given tuning is valid when collision count is zero; they only show that when collisions occur, the tuned certificate/controller is demonstrably invalid for the tested conditions.

The following simulation study examines the same distinction in a learning-based crowd navigation setting, where a CBF-derived reward improves empirical behavior but does not certify the learned policy.

\section{Crowd-Navigation Simulation Study}
\label{sec:crowd_navigation}

To supplement the distinction in \autoref{sec:empirical_safety}, I present a crowd-navigation comparison. This study demonstrates that CBF-based reward shaping may improve observed behavior without providing a formal safety certificate for the learned policy.

\subsection{Experimental Setup}

Two PPO-GNN policies are trained. The \emph{Base} policy follows the general state and reward construction used in CrowdNav~\cite{chen_crowdrobot_2019}, with a GNN encoder used to aggregate observations from multiple obstacles~\cite{kim_learning_2025a}. The \emph{CBF-RL} policy uses the same training procedure and adds a CBF-derived auxiliary reward motivated by~\cite{yang_cbfrl_2026}. Each trained policy is evaluated both directly and through the same inference-time CBF-QP, yielding four deployment conditions: Base, CBF-RL, Base + CBF-QP, and CBF-RL + CBF-QP.

The robot is a planar single integrator with control interval $\Delta t=0.1$ s and is evaluated among 10-20 dynamic obstacles. Each condition and crowd size uses 500 held-out episodes.

\begin{figure*}[t]
    \centering
    \includegraphics[width=0.98\textwidth]{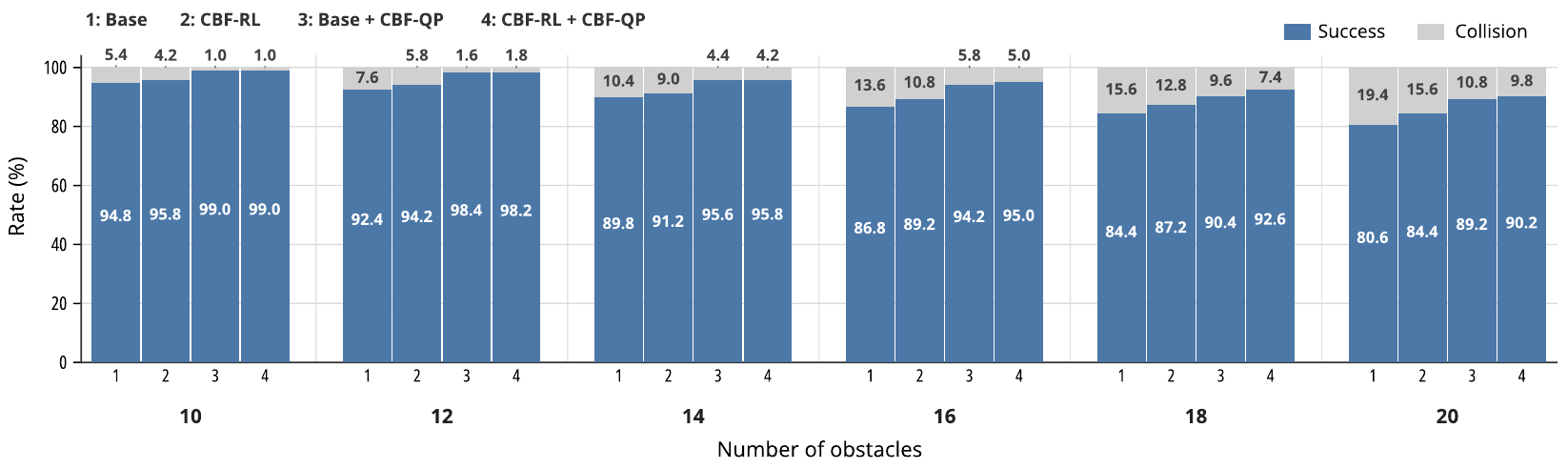}
    \caption{Empirical success and collision rates versus crowd size. Base and CBF-RL are evaluated with and without the same inference-time CBF-QP over 500 held-out episodes per setting; ``obstacles'' denotes active dynamic obstacles. Success and collision are independently recorded.}
    \label{fig:crowd_nav_results}
\end{figure*}

\subsection{Results and Discussion}

As shown in \autoref{fig:crowd_nav_results}, all four conditions degrade as crowd size increases. Without inference-time filtering, CBF-RL gives a modest point-estimate improvement over Base. At 20 obstacles, success increases from 80.6\% to 84.4\%, while collision decreases from 19.4\% to 15.6\%. The larger empirical change appears after applying the CBF-QP: Base + CBF-QP records 89.2\% success and 10.8\% collision, while CBF-RL + CBF-QP records 90.2\% success and 9.8\% collision.

The CBF reward is therefore \emph{safety-informed}, not \emph{safety-critical}. It penalizes barrier-condition violation during training, but it does not remove unsafe actions from the learned policy at inference. The QP provides the larger empirical collision reduction, yet nonzero collisions remain under both filtered conditions. This is consistent with the discussion above: the implemented QP with multiple simultaneous constraints can compete or become incompatible under bounded actuation, and sampled-data execution introduces model mismatch, numerical tolerances, and discretization effects. 

Interested readers may interactively compare the four deployment conditions through the web demo shown in \autoref{fig:crowd_nav_demo}.

\begin{figure}[t]
    \centering
    \includegraphics[width=0.9\columnwidth]{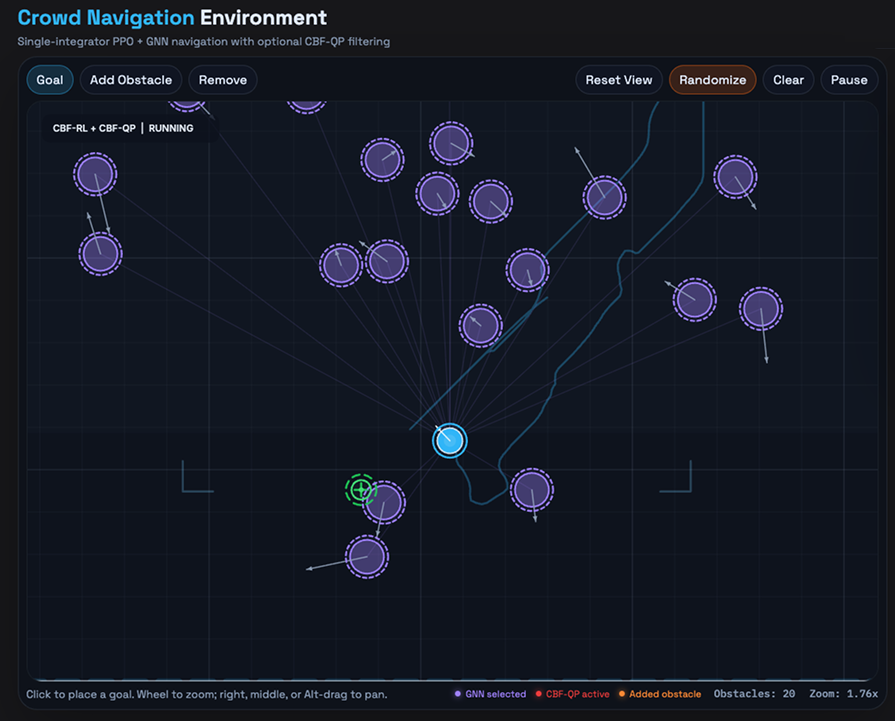}
    \caption{\url{https://cbf.taekyung.me}: interactive crowd-navigation demo for comparing the four deployment conditions.}
    \label{fig:crowd_nav_demo}
\end{figure}

\section{Conclusion}
This tutorial paper highlights a specific but common pitfall in safety-critical robotics: safety theorems are often invoked without verifying their hypotheses, yielding tautological guarantees. I emphasized (i) the distinction between candidate and valid CBFs, (ii) the central role of feasibility under true actuation limits, and (iii) the fact that safety demonstrations on passively safe (driftless) abstractions provide limited evidence for general robotic systems. The crowd-navigation study further shows that CBF-derived reward shaping can improve observed success and reduce collision rates, but cannot guarantee forward invariance without an inference-time safety filter. I advocate for future work that explicitly constructs or verifies controlled-invariant safe sets rather than implicitly assuming their existence.

\addtolength{\textheight}{0 cm}   




\bibliographystyle{IEEEtran}
\typeout{}
\bibliography{references.bib}

@inproceedings{yang_cbfrl_2026,
	title = {{CBF}-{RL}: {Safety} {Filtering} {Reinforcement} {Learning} in {Training} with {Control} {Barrier} {Functions}},
	shorttitle = {{CBF}-{RL}},
	doi = {10.48550/arXiv.2510.14959},
	urldate = {2025-10-17},
	booktitle = {{IEEE} {International} {Conference} on {Robotics} and {Automation} ({ICRA})},
	author = {Yang, Lizhi and Werner, Blake and Ames, Massimiliano de Sa Aaron D.},
	year = {2026},
}

@inproceedings{kim_learning_2025a,
	title = {Learning to {Adapt} {Control} {Barrier} {Functions} {Under} {Epistemic} and {Aleatoric} {Uncertainty}},
	language = {en},
	urldate = {2026-06-05},
	booktitle = {{arXiv} preprint {arXiv}.2504.03038},
	author = {Kim, Taekyung and Kee, Robin Inho and Panagou, Dimitra},
	year = {2026},
}

@inproceedings{parwana_ratetunable_2025,
	title = {Rate-{Tunable} {Control} {Barrier} {Functions}: {Methods} and {Algorithms} for {Online} {Adaptation}},
	issn = {2378-5861},
	shorttitle = {Rate-{Tunable} {Control} {Barrier} {Functions}},
	doi = {10.23919/ACC63710.2025.11108110},
	urldate = {2026-04-20},
	booktitle = {American {Control} {Conference} ({ACC})},
	author = {Parwana, Hardik and Panagou, Dimitra},
	year = {2025},
	pages = {275--282},
}

@article{breger_safe_2012,
	title = {Safe {Trajectories} for {Autonomous} {Rendezvous} of {Spacecraft}},
	doi = {10.2514/1.29590},
	language = {en},
	urldate = {2026-01-21},
	journal = {Journal of Guidance, Control, and Dynamics},
	author = {Breger, Louis and How, Jonathan P.},
	year = {2012},
}

@inproceedings{starek_sampling_2015,
	title = {A {Sampling} {Based} {Approach} to {Spacecraft} {Autonomous} {Maneuvering} with {Safety} {Specifications}},
	urldate = {2026-01-21},
	booktitle = {Annual {AAS} {Guidance} \& {Control} {Conference}},
	author = {Starek, Joseph A. and Barbee, Brent W. and Pavone, Marco},
	year = {2015},
}

@article{diamond_cvxpy_2016,
	title = {{CVXPY}: {A} {Python}-{Embedded} {Modeling} {Language} for {Convex} {Optimization}},
	language = {en},
	journal = {Journal of Machine Learning Research},
	author = {Diamond, Steven and Boyd, Stephen},
	year = {2016},
}

@inproceedings{park_collision_2026,
	title = {Beyond {Collision} {Cones}: {Dynamic} {Obstacle} {Avoidance} for {Nonholonomic} {Robots} via {Dynamic} {Parabolic} {Control} {Barrier} {Functions}},
	shorttitle = {{DPCBF}},
	doi = {10.48550/arXiv.2510.01402},
	urldate = {2025-12-15},
	booktitle = {International {Conference} on {Robotics} and {Automation} ({ICRA})},
	author = {Park, Hun Kuk and Kim, Taekyung and Panagou, Dimitra},
	year = {2026},
}

@article{xu_correctness_2018,
	title = {Correctness {Guarantees} for the {Composition} of {Lane} {Keeping} and {Adaptive} {Cruise} {Control}},
	volume = {15},
	issn = {1558-3783},
	doi = {10.1109/TASE.2017.2760863},
	number = {3},
	urldate = {2025-12-15},
	journal = {IEEE Transactions on Automation Science and Engineering},
	author = {Xu, Xiangru and Grizzle, Jessy W. and Tabuada, Paulo and Ames, Aaron D.},
	year = {2018},
	pages = {1216--1229},
}

@article{parrilo_semidefinite_2003,
	title = {Semidefinite programming relaxations for semialgebraic problems},
	volume = {96},
	issn = {1436-4646},
	doi = {10.1007/s10107-003-0387-5},
	language = {en},
	number = {2},
	urldate = {2025-12-15},
	journal = {Mathematical Programming},
	author = {Parrilo, Pablo A.},
	year = {2003},
	pages = {293--320},
}

@article{altarovici_general_2013,
	title = {A general {Hamilton}-{Jacobi} framework for non-linear state-constrained control problems},
	volume = {19},
	copyright = {© EDP Sciences, SMAI, 2012},
	issn = {1292-8119, 1262-3377},
	doi = {10.1051/cocv/2012011},
	language = {en},
	number = {2},
	urldate = {2025-12-01},
	journal = {ESAIM: Control, Optimisation and Calculus of Variations},
	author = {Altarovici, Albert and Bokanowski, Olivier and Zidani, Hasnaa},
	year = {2013},
	pages = {337--357},
}

@inproceedings{kim_how_2025,
	title = {How to {Adapt} {Control} {Barrier} {Functions}? {A} {Learning}-{Based} {Approach} with {Applications} to a {VTOL} {Quadplane}},
	shorttitle = {How to {Adapt} {Control} {Barrier} {Functions}?},
	doi = {10.48550/arXiv.2504.03038},
	urldate = {2025-08-22},
	booktitle = {{IEEE} {Conference} on {Decision} and {Control} ({CDC})},
	author = {Kim, Taekyung and Beard, Randal W. and Panagou, Dimitra},
	year = {2025},
	pages = {7050--7057},
}

@article{agrawal_reformulations_2025,
	title = {Reformulations of {Quadratic} {Programs} for {Lipschitz} {Continuity}},
	volume = {9},
	language = {en},
	journal = {IEEE Control Systems Letters},
	author = {Agrawal, Devansh R and Lee, Haejoon and Panagou, Dimitra},
	year = {2025},
	pages = {2603--2608},
}

@inproceedings{chen_crowdrobot_2019,
	title = {Crowd-{Robot} {Interaction}: {Crowd}-aware {Robot} {Navigation} with {Attention}-based {Deep} {Reinforcement} {Learning}},
	shorttitle = {{CrowdNav} or {SARL}},
	doi = {10.48550/arXiv.1809.08835},
	urldate = {2025-06-27},
	booktitle = {{IEEE} {International} {Conference} on {Robotics} and {Automation} ({ICRA})},
	author = {Chen, Changan and Liu, Yuejiang and Kreiss, Sven and Alahi, Alexandre},
	year = {2019},
	pages = {6015--6022},
}

@article{brunke_safe_2022,
	title = {Safe {Learning} in {Robotics}: {From} {Learning}-{Based} {Control} to {Safe} {Reinforcement} {Learning}},
	volume = {5},
	issn = {2573-5144},
	shorttitle = {Safe {Learning} in {Robotics}},
	doi = {10.1146/annurev-control-042920-020211},
	language = {en},
	urldate = {2025-03-22},
	journal = {Annual Review of Control, Robotics, and Autonomous Systems},
	author = {Brunke, Lukas and Greeff, Melissa and Hall, Adam W. and Yuan, Zhaocong and Zhou, Siqi and Panerati, Jacopo and Schoellig, Angela P.},
	year = {2022},
	pages = {411--444},
}

@article{kim_visibilityaware_2025,
	title = {Visibility-{Aware} {RRT}* for {Safety}-{Critical} {Navigation} of {Perception}-{Limited} {Robots} in {Unknown} {Environments}},
	volume = {10},
	doi = {10.48550/arXiv.2406.07728},
	number = {5},
	urldate = {2025-02-21},
	journal = {IEEE Robotics and Automation Letters},
	author = {Kim, Taekyung and Panagou, Dimitra},
	year = {2025},
	pages = {4508--4515},
}

@inproceedings{kim_learning_2025,
	title = {Learning to {Refine} {Input} {Constrained} {Control} {Barrier} {Functions} via {Uncertainty}-{Aware} {Online} {Parameter} {Adaptation}},
	doi = {10.48550/arXiv.2409.14616},
	urldate = {2025-02-17},
	booktitle = {{IEEE} {International} {Conference} on {Robotics} and {Automation} ({ICRA})},
	author = {Kim, Taekyung and Kee, Robin Inho and Panagou, Dimitra},
	year = {2025},
	pages = {3868--3875},
}

@inproceedings{choi_robust_2021,
	title = {Robust {Control} {Barrier}–{Value} {Functions} for {Safety}-{Critical} {Control}},
	shorttitle = {{CBVF}},
	doi = {10.1109/CDC45484.2021.9683085},
	urldate = {2024-09-09},
	booktitle = {{IEEE} {Conference} on {Decision} and {Control} ({CDC})},
	author = {Choi, Jason J. and Lee, Donggun and Sreenath, Koushil and Tomlin, Claire J. and Herbert, Sylvia L.},
	year = {2021},
	pages = {6814--6821},
}

@inproceedings{agrawal_safe_2021,
	title = {Safe {Control} {Synthesis} via {Input} {Constrained} {Control} {Barrier} {Functions}},
	shorttitle = {{ICCBF}},
	doi = {10.1109/CDC45484.2021.9682938},
	urldate = {2024-07-02},
	booktitle = {{IEEE} {Conference} on {Decision} and {Control} ({CDC})},
	author = {Agrawal, Devansh R. and Panagou, Dimitra},
	year = {2021},
	pages = {6113--6118},
}

@book{borrelli_predictive_2017,
	title = {Predictive {Control} for {Linear} and {Hybrid} {Systems}},
	isbn = {978-1-107-01688-0},
	language = {en},
	publisher = {Cambridge University Press},
	author = {Borrelli, Francesco and Bemporad, Alberto and Morari, Manfred},
	year = {2017},
}

@article{xiong_discretetime_2023,
	title = {Discrete-{Time} {Control} {Barrier} {Function}: {High}-{Order} {Case} and {Adaptive} {Case}},
	volume = {53},
	copyright = {https://ieeexplore.ieee.org/Xplorehelp/downloads/license-information/IEEE.html},
	issn = {2168-2267, 2168-2275},
	shorttitle = {dt-hocbf},
	doi = {10.1109/TCYB.2022.3170607},
	language = {en},
	number = {5},
	urldate = {2024-04-04},
	journal = {IEEE Transactions on Cybernetics},
	author = {Xiong, Yuhan and Zhai, Di-Hua and Tavakoli, Mahdi and Xia, Yuanqing},
	year = {2023},
	pages = {3231--3239},
}

@inproceedings{singletary_comparative_2021,
	title = {Comparative {Analysis} of {Control} {Barrier} {Functions} and {Artificial} {Potential} {Fields} for {Obstacle} {Avoidance}},
	doi = {10.1109/IROS51168.2021.9636670},
	urldate = {2024-02-07},
	booktitle = {{IEEE}/{RSJ} {International} {Conference} on {Intelligent} {Robots} and {Systems} ({IROS})},
	author = {Singletary, Andrew and Klingebiel, Karl and Bourne, Joseph and Browning, Andrew and Tokumaru, Phil and Ames, Aaron},
	year = {2021},
	pages = {8129--8136},
}

@inproceedings{xiao_control_2019,
	title = {Control {Barrier} {Functions} for {Systems} with {High} {Relative} {Degree}},
	shorttitle = {hocbf},
	doi = {10.1109/CDC40024.2019.9029455},
	urldate = {2023-12-05},
	booktitle = {{IEEE} {Conference} on {Decision} and {Control} ({CDC})},
	author = {Xiao, Wei and Belta, Calin},
	year = {2019},
	pages = {474--479},
}

@inproceedings{ames_control_2019,
	title = {Control {Barrier} {Functions}: {Theory} and {Applications}},
	shorttitle = {{CBF}},
	doi = {10.23919/ECC.2019.8796030},
	booktitle = {European {Control} {Conference} ({ECC})},
	author = {Ames, Aaron D. and Coogan, Samuel and Egerstedt, Magnus and Notomista, Gennaro and Sreenath, Koushil and Tabuada, Paulo},
	year = {2019},
	pages = {3420--3431},
}

\end{document}